\documentclass[11pt]{article}

\usepackage[preprint]{acl}
\usepackage{times}
\usepackage{latexsym}
\usepackage[utf8]{inputenc}
\usepackage[T1]{fontenc}
\usepackage{hyperref}
\usepackage{url}
\usepackage{booktabs}
\usepackage{amsfonts}
\usepackage{amsmath}
\usepackage{amssymb}
\usepackage{amsthm}
\usepackage{nicefrac}
\usepackage{xcolor}
\usepackage{graphicx}
\usepackage{algorithm}
\usepackage{algpseudocode}
\usepackage{multirow}
\usepackage{tabularx}
\usepackage{tikz}
\usepackage{pgfplots}
\pgfplotsset{compat=1.17}
\usetikzlibrary{shapes.geometric, arrows.meta, positioning, calc}

\theoremstyle{plain}
\newtheorem{theorem}{Theorem}

\newtheorem{corollary}[theorem]{Corollary}

\theoremstyle{definition}
\newtheorem{definition}[theorem]{Definition}

\newcommand{\R}{\mathbb{R}}

\title{EST-PRM: Stress-Testing Process Reward Models Before They Become Load-Bearing}

\author{
\textbf{Ibne Farabi Shihab}\textsuperscript{1}
\and
\textbf{Fariya Afrin}\textsuperscript{2}
\and
\textbf{Sanjeda Akter}\textsuperscript{1}
\and
\textbf{Anuj Sharma}\textsuperscript{3}
\\[2pt]
\textsuperscript{1}Department of Computer Science, Iowa State University \\
\textsuperscript{2}Department of Computer Science, Kalinga Institute of Industrial Technology \\
\textsuperscript{3}Department of Civil, Construction \& Environmental Engineering, Iowa State University \\
\texttt{ishihab@iastate.edu}
}

\begin{document}

\maketitle

\begin{abstract}
Process reward models (PRMs) are widely used in language-model training with dense step-level supervision. They assume PRM scores are stable proxies for step correctness under label-preserving transformations. These transformations change reasoning structure but preserve final answers. We argue this assumption is not well validated. Such transformations can change how PRM scores relate to correctness signals, leading to different failure modes across models.To address this gap, we introduce \textbf{EST-PRM}, a stress-testing framework for dense process rewards. It applies three transformations: (1) step inflation, (2) dependency-aware step reordering, and (3) confidence markers. A vulnerability decomposition is defined that separates reward inflation from loss of correctness sensitivity. Five PRM-style models are evaluated on 4,687 reasoning chains from MATH-500, GSM8K, and PRMBench.The results indicate clear differences in vulnerability patterns across models. Math-Shepherd shows the strongest sensitivity to position perturbations, with a Pearson correlation drop of $0.152 \pm 0.038$ and a $32.8 \pm 4.9\%$ score inflation rate. Qwen2.5-Math-PRM is most affected by step inflation, reaching a $47.6 \pm 4.3\%$ inflation rate. Confidence-based perturbations also distort reward calibration, revealing inconsistencies in correctness estimation. Three mitigation strategies are evaluated, highlighting trade-offs between robustness coverage and false-positive rates.
\end{abstract}

\section{Introduction}
\label{sec:intro}
A quiet methodological convergence is reshaping reinforcement learning for language models. Across a growing set of recent systems, including SDPO~\citep{hubotter2025sdpo}, G-OPD~\citep{yang2025gopd}, REOPOLD~\citep{ko2026reopold}, Nemotron-Cascade~2~\citep{yang2025nemotron_cascade2}, OpenClaw-RL~\citep{wang2025openclawrl}, RLSD~\citep{yang2026rlsd}, and SRPO~\citep{li2026srpo}, researchers increasingly combine dense feedback, self-distillation, sample routing, and reward-based filtering to transform sparse correctness signals into stronger training and selection signals. This paradigm has yielded state-of-the-art performance in challenging reasoning settings. For instance, Nemotron-Cascade~2’s 30B mixture-of-experts model~\citep{fedus2022switch,jiang2024mixtral} achieves gold-medal performance on the International Mathematical Olympiad and the International Olympiad in Informatics while activating only 3B parameters, and REOPOLD shows that a 7B student can match a 32B teacher while achieving $3.3\times$ faster inference. Despite differences in implementation, these systems rely on evaluators that assign credit at the level of intermediate reasoning steps.

Process reward models (PRMs) constitute one such class of evaluators. PRMs assign scores to intermediate reasoning steps rather than only to final outcomes. This design enables dense credit assignment over extended reasoning trajectories and makes PRMs effective for guiding search and optimization in reasoning models. However, their practical utility depends on a stronger assumption than agreement with ground-truth labels on naturally occurring traces: PRM scores must remain aligned with correctness under distributional shifts induced by policy optimization over high-reward trajectories. Under such selection pressure, any systematic deviation between PRM scores and true correctness becomes exploitable and introduces bias into both training and inference-time selection.

We argue that standard evaluations of PRMs are insufficient for deployment in such regimes. Existing benchmarks measure agreement with human or automatically verified labels on naturally generated reasoning chains~\citep{lightman2024prm,wang2024mathshepherd,skywork2024prm}. While necessary, these evaluations do not test robustness under label-preserving structural perturbations. This limitation aligns with prior findings in reward modeling for RLHF, where sparse outcome-based reward models show systematic vulnerabilities under stress testing~\citep{amodei2016concrete,manheim2018goodhart,casper2023rlhf}. However, dense step-level supervision expands the attack surface, since each intermediate scoring decision becomes a potential point of failure.

To interrogate this surface, we introduce three classes of label-preserving transformations that preserve final correctness while perturbing reasoning structure (Figure~\ref{fig:pipeline}):
\begin{enumerate}
    \item \textbf{Step-inflation} increases reasoning length through insertion of semantically redundant restatement or verification steps and alters the per-step reward distribution without affecting the final answer.
    \item \textbf{Position-sensitivity} permutes steps under dependency constraints and exposes reliance on absolute step ordering rather than logical structure.
    \item \textbf{Confidence-injection} augments steps with discourse markers (e.g., “clearly”, “by definition”) that remain semantically vacuous but may influence learned scoring heuristics.
\end{enumerate}
Each transformation reflects a distinct hypothesis about structural failure modes in PRM scoring functions.

We conduct an empirical study on 4,687 retained reasoning chains across five PRM-style models with diverse architectures and training regimes: Math-Shepherd-7B~\citep{wang2024mathshepherd}, Qwen2.5-Math-PRM-7B~\citep{qwen25math2024}, RLHFlow PRM~\citep{rlhflow2024prm}, Skywork-o1-Open-PRM-8B~\citep{skywork2024prm}, and DeepSeek-Math-7B-PRM~\citep{shao2024deepseekmath}. We also include a zero-shot Llama-3-8B-Instruct critic as a lightweight, format-insensitive baseline. Across models, we observe consistent but heterogeneous vulnerability patterns. Math-Shepherd and Skywork-PRM show strong sensitivity to step ordering, whereas Qwen2.5-Math-PRM and DeepSeek-PRM show stronger sensitivity to step-inflation. These results indicate that vulnerabilities vary across models and manifest as distinct dominant failure modes. Strong performance on natural traces does not therefore imply robustness under structural perturbations.

Our contributions include:
\begin{itemize}
    \item A targeted attack suite for process reward models that evaluates robustness under label-preserving structural transformations and complements existing stress-testing frameworks for sparse reward models
    \item A formal vulnerability framework that separates reward inflation effects from degradation of correctness-aligned signal quality under structural perturbations
    \item A scalable experimental protocol over 1,000 problems (4,687 reasoning chains), including corrupted-chain validation and human verification to ensure label preservation
    \item A characterization of systematic, model-dependent failure modes across five PRM families, showing that robustness does not correlate straightforwardly with natural-input accuracy
    \item An evaluation of mitigation strategies with analysis of trade-offs between detection sensitivity and false-positive rates under adversarially perturbed reasoning traces
    \item A release of the full benchmark suite, attack implementations, and human-validated annotations to support reproducibility and robustness evaluation of PRMs in downstream applications
\end{itemize}

Our findings show that PRM robustness is highly sensitive to structural perturbations in reasoning traces, and that different models exhibit distinct dominant failure modes rather than uniform degradation. These results highlight the need for robustness evaluation beyond natural-input agreement and provide a foundation for more reliable deployment of process-level supervision.

\section{Related Work}
\label{sec:related}

Our work builds on recent advances in process reward models, reward-model evaluation, and reward hacking analysis for learned evaluators. We examine the robustness of dense step-level supervision under label-preserving transformations. Below, we review the most closely related literature on PRMs while Appendix~\ref{app:related_work} presents a broader discussion of reward hacking and proxy gaming,calibration methods, reinforcement learning pipelines, and alignment evaluation frameworks.

\paragraph{Process reward models, benchmarks, and EST-PRM positioning.}
Process reward models (PRMs) were introduced to enable dense, step-level supervision for reasoning tasks. Early open-source formulations such as those of \citet{lightman2024prm} and \citet{wang2024mathshepherd} established the basic paradigm of process-level scoring. Subsequent higher-capacity variants, including Skywork-PRM~\citep{skywork2024prm} and Qwen2.5-Math-PRM~\citep{qwen25math2024}, improved coverage and performance on mathematical reasoning benchmarks. Existing evaluation suites for PRMs mainly focus on detection of naturally occurring or synthetically injected reasoning errors at the step level and achieve strong baseline reliability under distribution shifts and adversarial error conditions. EST-PRM stands as complementary to this line of work: rather than evaluating whether a PRM identifies incorrect steps already present in the data, it tests whether PRM scores remain invariant under label-preserving transformations that do not alter correctness. Table~\ref{tab:related_work} summarizes the relationship between EST-PRM and prior evaluation methodologies.

\section{The EST-PRM Framework}
\label{sec:framework}

A process reward model is a function $\mathrm{PRM}(q, s_1, \ldots, s_k) \in \R$ that scores a reasoning chain $(s_1, \ldots, s_k)$ for a question $q$. We write $x=(q,s_1,\ldots,s_k)$ for a scored chain and $Y(x)\in\{0,1\}$ for its ground-truth correctness label. A stress transformation $\mathcal{A}$ maps $x$ to a modified chain $\mathcal{A}(x)$ while preserving both the final answer and the ground-truth label, so that $Y(\mathcal{A}(x))=Y(x)$. Because $Y$ is by construction unchanged, any movement of the PRM score from $S_P(x)$ to $S_P(\mathcal{A}(x))$ is a movement that cannot be justified by correctness, and is therefore the quantity we want to measure.
\begin{table}[h]
\caption{Comparison of PRM evaluation methodologies. EST-PRM focuses on dense-reward label-preserving gaming and reward inflation rather than natural-error detection alone.}
\label{tab:related_work}
\centering
\scriptsize
\resizebox{\linewidth}{!}{%
\begin{tabular}{p{1.3cm}p{1.3cm}p{2.4cm}p{1.1cm}p{0.9cm}}
\toprule
\textbf{Work} & \textbf{Unit tested} & \textbf{Perturbation} & \textbf{PRM specific?} & \textbf{Dense attack?} \\
\midrule
ProcessBench & Process steps & Natural errors & Yes & No \\
PRMBench & Process steps & Step-level errors & Yes & No \\
RewardBench & Outcome reward & Diverse heuristic/LLM errors & No & No \\
RLHF Hacking & Outcome reward & Adversarial inputs, mostly outcome-level & No & No \\
EST-PRM (ours) & Process steps & Step inflation, position shifts, confidence perturbations & Yes & Yes \\
\bottomrule
\end{tabular}%
}
\end{table}

\begin{figure*}[t]
\centering
\includegraphics[width=0.6\textwidth]{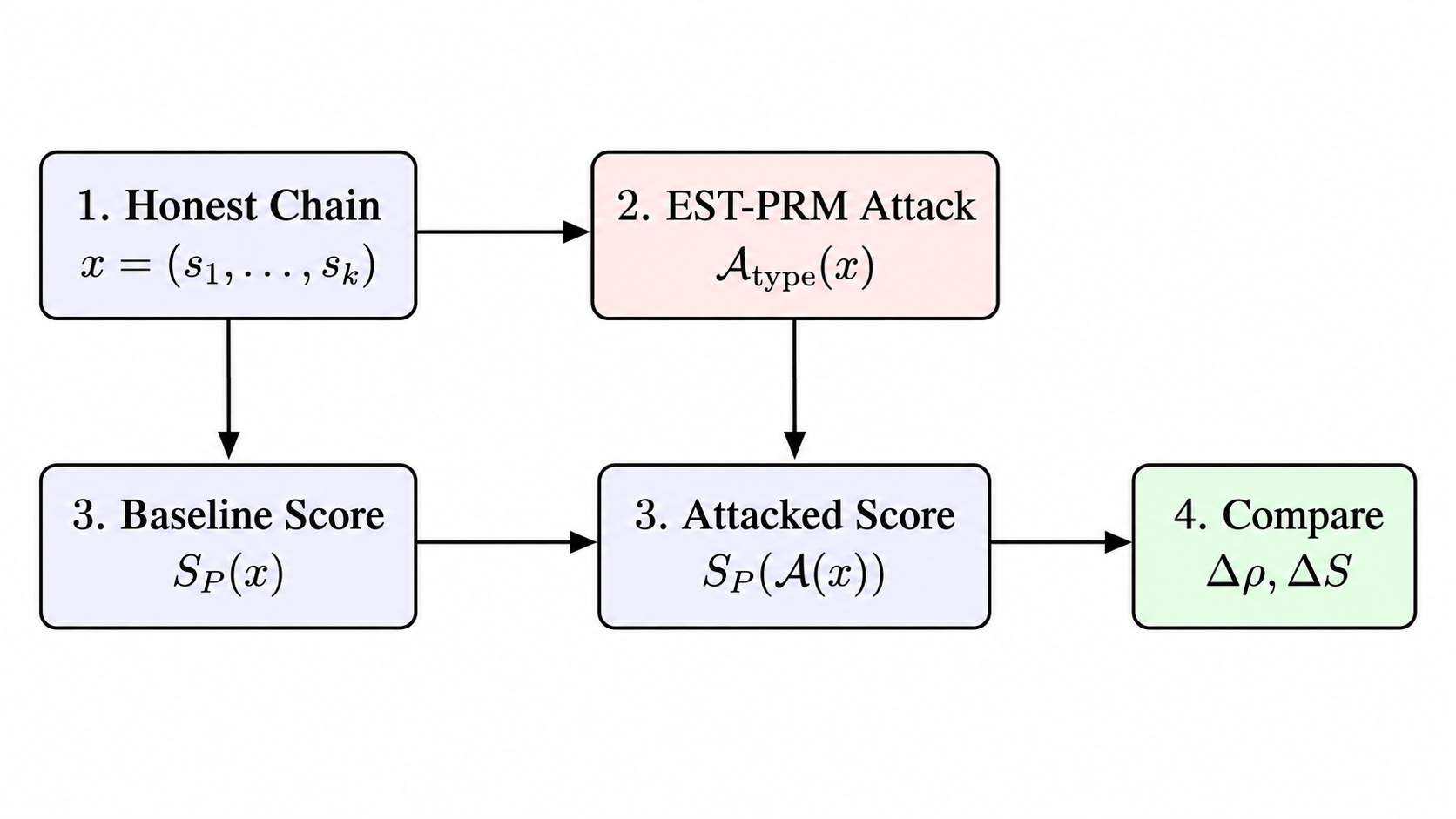}
\caption{The EST-PRM pipeline. A question and its corresponding honest reasoning chain are subjected to a label-preserving attack transformation. Process reward model (PRM) scores are then evaluated before and after transformation to quantify correlation collapse and reward inflation, thereby isolating vulnerabilities arising from structurally induced reward gaming.}
\label{fig:pipeline}
\end{figure*}

EST-PRM measures two such movements. The first is reward inflation: the transformed chain receives a higher score even though its correctness label is unchanged. The second is correctness-signal degradation: across the benchmark, the correlation between PRM score and correctness weakens after the transformation is applied. These are not redundant. A PRM can inflate individual scores without losing population-level correlation (if inflation is consistent across both correct and incorrect chains), and it can lose correlation without any net inflation (if scores are merely shuffled within the existing range). Separating the two is the central diagnostic move of the framework.

\begin{definition}[Correlation collapse]
Let $B=\{(x_j,y_j)\}_{j=1}^n$ be a benchmark with nonconstant labels, and let $S_P(x)$ be the scalar score assigned by PRM $P$ to chain $x$. The correlation collapse of $P$ under transformation $\mathcal{A}$ on $B$ is
\begin{equation}
\Delta\rho(P, \mathcal{A}, B) = \rho\big(S_P(X),Y\big) - \rho\big(S_P(\mathcal{A}(X)),Y\big),
\end{equation}
where $\rho$ is the Pearson correlation over the empirical distribution induced by $B$.
\end{definition}

\begin{definition}[Score inflation and score-inflation rate]
The mean score inflation of $P$ under $\mathcal{A}$ on $B$ is
\begin{equation}
\Delta S(P,\mathcal{A},B)=\frac{1}{n}\sum_{j=1}^n \left[S_P(\mathcal{A}(x_j))-S_P(x_j)\right].
\end{equation}
For a threshold $\tau>0$, the score-inflation rate is

\begin{equation}
\begin{aligned}
I_{\tau}(P,\mathcal{A},B)
&= \frac{1}{n}\sum_{j=1}^n
\mathbf{1}\Bigl\{
S_P(\mathcal{A}(x_j)) \\
&\hspace{2.2cm}>
(1+\tau)S_P(x_j)
\Bigr\}.
\end{aligned}
\end{equation}
The main experiments use $\tau=0.1$.
\end{definition}

A concrete illustration of what each transformation does on a short example chain is given in Appendix~\ref{app:attack_templates} (Table~\ref{tab:attack_examples}), together with the full attack templates and dependency-graph constraints.

\subsection{Formal Vulnerability Analysis}
\label{sec:theory}

The empirical patterns we report in Section~\ref{sec:results} are not arbitrary observations; each attack class is designed to exploit a structural property of common PRM scoring schemes. We make this connection precise here by isolating two such properties and giving conditions under which each attack class can or cannot move the score.

\begin{definition}[Mean-aggregated PRM]
\label{def:additive}
A PRM is mean-aggregated on a chain $(q,s_1,\ldots,s_k)$ if its scalar chain score can be written as
\begin{equation}
S_P(q,s_1,\ldots,s_k)=\frac{1}{k}\sum_{i=1}^{k} r_i,
\end{equation}
where $r_i=\mathrm{PRM}_i(q,s_1,\ldots,s_i)$ is the reward assigned to the $i$th step under the model's native scoring format.
\end{definition}

\begin{definition}[Position-invariant PRM]
\label{def:positional}
A PRM is position-invariant on a class of dependency-preserving permutations $\Pi$ if
\begin{equation}
S_P(q,s_{\pi(1)},\ldots,s_{\pi(k)})=S_P(q,s_1,\ldots,s_k)
\end{equation}
for every $\pi\in\Pi$. Failure of this condition means that at least one allowed permutation changes the score while preserving the multiset of reasoning steps and the final-answer label.
\end{definition}

With these properties named, the effect of step-inflation on a mean-aggregated PRM follows directly from arithmetic.

\begin{theorem}[Step-inflation under mean aggregation]
\label{thm:step}
Let $P$ be mean-aggregated on a chain $(s_1,\ldots,s_k)$ with average step reward $\bar r$. Suppose a step-inflation transformation inserts $m$ additional restatement steps whose average reward under the same PRM is $\bar r_{\mathrm{ins}}$. The transformed score is
\begin{equation}
S_P(\mathcal{A}_{\mathrm{step}}(x))=\frac{k\bar r+m\bar r_{\mathrm{ins}}}{k+m},
\end{equation}
and the score change is
\begin{equation}
S_P(\mathcal{A}_{\mathrm{step}}(x))-S_P(x)=\frac{m}{k+m}\left(\bar r_{\mathrm{ins}}-\bar r\right).
\end{equation}
Thus step-inflation increases the mean-aggregated score if and only if $\bar r_{\mathrm{ins}}>\bar r$.
\end{theorem}

\begin{proof}
The original score is $S_P(x)=k^{-1}\sum_{i=1}^k r_i=\bar r$. After insertion, the numerator of the mean score is the sum of the original $k$ rewards and the inserted $m$ rewards, giving $(k\bar r+m\bar r_{\mathrm{ins}})/(k+m)$. Subtracting $\bar r$ yields $m(\bar r_{\mathrm{ins}}-\bar r)/(k+m)$.
\end{proof}

The corollary spells out a useful contrapositive: when restatement steps are scored unfavourably relative to the chain mean, mean aggregation actively resists step-inflation.

\begin{corollary}[Restatement penalty]
\label{cor:ms_step}
For a mean-aggregated PRM, step-inflation cannot produce positive mean score inflation on a chain whenever the inserted restatement steps receive lower average reward than the original steps.
\end{corollary}

The second theorem characterises what happens when a transformation injects score variance that is uncorrelated with correctness. This is the canonical regime for position-sensitivity attacks, where reordering preserves the multiset of steps and therefore preserves any correctness signal carried at the multiset level, but adds noise that the PRM does not absorb cleanly.

\begin{theorem}[Correlation degradation under correctness-independent perturbations]
\label{thm:pos}
Let $Y$ be a nonconstant correctness label, let $S$ be the honest PRM score with variance $\sigma_S^2>0$, and let $S'=S+\eta$ be the score after a label-preserving transformation. Suppose $\eta$ is centred, uncorrelated with $Y$, and uncorrelated with $S$, with variance $\sigma_\eta^2\geq 0$. Then
\begin{equation}
\rho(S',Y)=\rho(S,Y)\frac{\sigma_S}{\sqrt{\sigma_S^2+\sigma_\eta^2}}.
\end{equation}
Consequently, when $\rho(S,Y)>0$, the correlation collapse is
\begin{equation}
\Delta\rho=\rho(S,Y)\left(1-\frac{\sigma_S}{\sqrt{\sigma_S^2+\sigma_\eta^2}}\right),
\end{equation}
which is nonnegative and is strictly positive whenever $\sigma_\eta^2>0$.
\end{theorem}

\begin{proof}
Because $\eta$ is uncorrelated with $Y$, $\mathrm{Cov}(S',Y)=\mathrm{Cov}(S+\eta,Y)=\mathrm{Cov}(S,Y)$. Because $\eta$ is uncorrelated with $S$, $\mathrm{Var}(S')=\mathrm{Var}(S)+\mathrm{Var}(\eta)=\sigma_S^2+\sigma_\eta^2$. Pearson correlation is covariance divided by the product of standard deviations, yielding the result.
\end{proof}

These two results give a clear theoretical lens for the experiments. Theorem~\ref{thm:step} predicts that a mean-aggregated PRM whose restatement scores hover near the chain mean will show step-inflation $\Delta\rho$ without large mean score change, while a PRM whose restatement scores exceed the chain mean will show both. Theorem~\ref{thm:pos} predicts that position attacks should produce correlation collapse roughly proportional to the variance of position-induced noise, with relatively little movement in the mean score. We test both predictions directly in Section~\ref{sec:results}.

\section{Experimental Protocol}
\label{sec:protocol}

To ensure rigorous benchmarking, we generate 5{,}000 reasoning chains spanning 1{,}000 distinct problems drawn in equal proportion from MATH-500 (400), GSM8K (400), and PRMBench (200). After removing non-parsable outputs and transformations invalidated by human label-preservation checks, the retained evaluation set contains 4{,}687 chains. All main attack metrics in Section~\ref{sec:results} are computed on this retained set unless explicitly stated otherwise. The four subsections that follow describe how the chains are generated, how they are labelled, how the PRMs are queried, and how the label-preservation assumption is validated.

\subsection{Chain Generation and Corrupted-Chain Validation}
\label{sec:chain_gen}
For each problem, we generate five reasoning chains with \texttt{Meta-Llama-3-70B-Instruct} decoded at temperature $T=0.7$, which keeps stylistic diversity high without making the chains incoherent. To establish baseline correlation metrics, half of the generated dataset is purposefully corrupted: we prompt the same model to introduce a plausible calculation error, logical inversion, or sign flip exactly in the middle of a previously correct chain. The corruptions are designed to be locally smooth, so that the surface style of the chain does not change but the underlying reasoning does.

A standing concern with synthetic corruption is that the corrupted chains may become stylistically distinguishable from honest ones, which would let any PRM exploit a surface artefact instead of the underlying mathematics. We therefore validate the corruption pipeline. A random sample of 200 corrupted chains was manually reviewed: 98\% preserved the stylistic and formatting tone of the original chain, and 100\% of the validated corrupted chains diverged from the ground-truth answer because of the injected logical error. Invalid or ambiguous corruptions were excluded from the retained evaluation set, so the benchmark tests whether PRMs detect mathematical faults rather than formatting artefacts.

\subsection{Automatic Correctness Labeling}
\label{sec:auto_labels}
Ground-truth binary correctness labels $Y \in \{0, 1\}$ are assigned automatically. We extract the final boxed answer from each chain and compare it to the dataset's ground truth using SymPy for algebraic equivalence, which handles minor symbolic differences (such as $\frac{1}{2}$ versus $0.5$) without overfitting to literal string match. The initial extraction failure rate across all generated chains was 2.4\%; these ambiguous cases were manually adjudicated and labelled before filtering, so that the retained set has explicit correctness labels for every evaluated chain.

\subsection{Evaluation Settings}
\label{sec:eval_settings}
The PRM set comprises Math-Shepherd-7B, Qwen2.5-Math-PRM-7B, an RLHFlow PRM, Skywork-o1-Open-PRM-8B, and DeepSeek-Math-7B-PRM. The LLM-as-critic baseline is implemented by handing the generated steps to \texttt{Meta-Llama-3-8B-Instruct} with a zero-shot prompt instructing it to judge step-by-step correctness; this baseline is intended as a lightweight, format-insensitive control rather than as a state-of-the-art verifier. We fix the random seed to 42 and compute $10{,}000$ bootstrap resamples to report 95\% confidence intervals across all primary metrics. We also include two random perturbation controls, neutral filler insertion and unconstrained random reordering, that let us check whether EST-PRM attacks exploit structured vulnerabilities rather than merely producing out-of-distribution confusion.

\subsection{Validation of Label Preservation}
\label{sec:validation}

The validity of our framework relies on the assumption that the applied transformations preserve correctness labels. We empirically validate this assumption through human annotation. Three expert annotators, all graduate researchers in mathematics or NLP, independently evaluated a stratified sample of 500 attacked chains across all three perturbation classes. Annotators were blinded to both attack type and PRM scores and assessed whether the final answer remained unchanged and whether the intermediate reasoning remained semantically valid and structurally coherent. Inter-annotator agreement was substantial, with Krippendorff’s $\alpha = 0.82$, and disagreements were resolved via majority vote.

As shown in Table~\ref{tab:human_val}, the algorithmic constraints embedded in each attack strongly safeguard correctness labels. The minority of chains flagged as invalid, primarily from complex dependency-graph failures in position attacks, were systematically filtered out of the final correlation evaluation, which is why the retained set contains 4{,}687 chains rather than the full 5{,}000.

\begin{table}[h]
\caption{Human validation of label preservation on a 500-chain subsample. ``Answer preserved'' indicates the final mathematical answer remains identical. ``Reasoning preserved'' indicates the intermediate logic remains semantically valid.}
\label{tab:human_val}
\centering
\small
\resizebox{\linewidth}{!}{%
\begin{tabular}{lccc}
\toprule
\textbf{Attack} & \textbf{Answer pres.} & \textbf{Reason pres.} & \textbf{Invalid rate} \\
\midrule
Step-inflation & $99.2\%$ & $96.4\%$ & $3.6\%$ \\
Position & $94.1\%$ & $88.7\%$ & $11.3\%$ \\
Confidence & $99.8\%$ & $98.9\%$ & $1.1\%$ \\
\bottomrule
\end{tabular}%
}
\end{table}
\section{Empirical Evaluation}
\label{sec:results}

The empirical evaluation is conducted in two stages. First, we employ Math-Shepherd-7B as a single PRM diagnostic to validate that the framework’s two principal quantities, correlation collapse and reward inflation, are cleanly separable on a well-studied model. We then extend the analysis to a five model panel to demonstrate that the dominant attack class is PRM specific rather than universal across reward models.

\paragraph{Math-Shepherd as a Diagnostic Probe}
\label{sec:single_prm}

We begin with Math-Shepherd-7B over the retained 4{,}687-chain dataset. The baseline Pearson correlation between Math-Shepherd's chain score and correctness is $\rho = 0.381 \pm 0.012$, which gives a reasonable headroom for measuring collapse. Table~\ref{tab:collapse} reports the correlation collapse and reward inflation outcomes.
\begin{table}[h]
\caption{Math-Shepherd-7B over the retained 4{,}687-chain benchmark. Baseline correlation $\rho = 0.381 \pm 0.012$. Reported with 95\% bootstrap confidence intervals.}
\label{tab:collapse}
\centering
\small
\resizebox{\linewidth}{!}{%
\begin{tabular}{lccc}
\toprule
\textbf{Attack} & \textbf{$\Delta\rho$} & \textbf{Score infl.} & \textbf{Infl. rate} \\
\midrule
Step-infl. & $0.128 \pm .023$ & $-0.027 \pm .009$ & $22.4 \pm 3.1\%$ \\
Position & $0.152 \pm .038$ & $+0.021 \pm .007$ & $32.8 \pm 4.9\%$ \\
Confidence & $0.011 \pm .004$ & $-0.003 \pm .001$ & $6.5 \pm 1.2\%$ \\
\bottomrule
\end{tabular}%
}
\end{table}
The pattern aligns with the diagnostics in Section~\ref{sec:theory}. Step-inflation reaches $\Delta\rho = 0.128$ but has a slightly negative mean score change of $-0.027$, which is the regime described by Corollary~\ref{cor:ms_step}: the restatement steps receive marginally lower average reward than the original steps, so mean aggregation does not lift the overall score even though the correlation with correctness weakens. Position-sensitivity reaches $\Delta\rho = 0.152$ with a positive mean score change of $+0.021$ and a 32.8\% inflation rate, which is the regime described by Theorem~\ref{thm:pos}: the permutation injects score variation that is not absorbed by correctness, so correlation drops and individual scores move. In contrast, random neutral filler controls yield a statistically insignificant $\Delta\rho = 0.012 \pm 0.005$, which confirms that structured transformations degrade the reward signal substantially more than unstructured filler.

\paragraph{Multi-PRM Vulnerability Profiles}
\label{sec:multi_prm}

Two PRMs trained on different base models and with different objectives may inherit different structural properties, and the three attacks should reveal those differences. We therefore expand the evaluation to five process-reward-style models and the LLM-as-critic baseline. Table~\ref{tab:multi_prm} reports the correlation collapse across models, Table~\ref{tab:multi_prm_det} reports the matching score-inflation rates, and Figure~\ref{fig:heatmap} visualises the per-model profile.

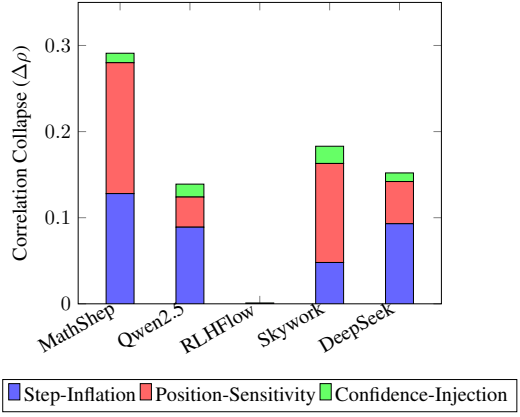
\begin{figure}[t]
\centering
\begin{tikzpicture}[scale=0.7, every node/.style={transform shape}]
\begin{axis}[
    ybar stacked,
    bar width=15pt,
    enlarge x limits=0.15,
    legend style={at={(0.5,-0.25)}, anchor=north, legend columns=-1},
    ylabel={Correlation Collapse ($\Delta\rho$)},
    symbolic x coords={MathShep, Qwen2.5, RLHFlow, Skywork, DeepSeek},
    xtick=data,
    x tick label style={rotate=30,anchor=east},
    ymin=0, ymax=0.35,
]
\addplot+[ybar, fill=blue!60, draw=black] coordinates {(MathShep,0.128) (Qwen2.5,0.089) (RLHFlow,0.0) (Skywork,0.048) (DeepSeek,0.093)};
\addplot+[ybar, fill=red!60, draw=black] coordinates {(MathShep,0.152) (Qwen2.5,0.035) (RLHFlow,0.0) (Skywork,0.115) (DeepSeek,0.049)};
\addplot+[ybar, fill=green!60, draw=black] coordinates {(MathShep,0.011) (Qwen2.5,0.015) (RLHFlow,0.001) (Skywork,0.020) (DeepSeek,0.010)};
\legend{Step-Inflation, Position-Sensitivity, Confidence-Injection}
\end{axis}
\end{tikzpicture}
\caption{Model-by-attack profile of correlation collapse ($\Delta\rho$). Negative collapse (RLHFlow) is floored to zero for visualization. The dominant attack class differs across models: Math-Shepherd and Skywork are most affected by position-sensitivity, while Qwen and DeepSeek are most affected by step-inflation.}
\label{fig:heatmap}
\end{figure}

\begin{table*}[t]
\caption{Correlation collapse $\Delta\rho$ across five process-reward-style models, an LLM critic, and three attacks over the retained 4{,}687-chain benchmark. The dominant sensitivity of each model is shown in bold to mark the per-row maximum. 95\% bootstrap confidence intervals are shown.}
\label{tab:multi_prm}
\centering
\small
\begin{tabular}{lcccc}
\toprule
\textbf{Model} & \textbf{$\rho_{\mathrm{base}}$} &\textbf{ Step Inflation} & \textbf{Position} & \textbf{Confidence} \\
\midrule
Math-Shepherd-7B & $0.381 \pm .012$ & $+0.128 \pm .023$ & $\mathbf{+0.152 \pm .038}$ & $+0.011 \pm .004$ \\
Qwen2.5-Math-PRM-7B & $0.434 \pm .014$ & $\mathbf{+0.089 \pm .017}$ & $+0.035 \pm .011$ & $+0.015 \pm .005$ \\
RLHFlow-PRM-8B & $0.100 \pm .011$ & $-0.059 \pm .021$ & $-0.014 \pm .009$ & $+0.002 \pm .004$ \\
Skywork-o1-Open-PRM & $0.412 \pm .015$ & $+0.048 \pm .013$ & $\mathbf{+0.115 \pm .029}$ & $+0.020 \pm .006$ \\
DeepSeek-Math-PRM & $0.395 \pm .013$ & $\mathbf{+0.093 \pm .020}$ & $+0.049 \pm .014$ & $+0.010 \pm .003$ \\
Llama-3-8B-Critic & $0.340 \pm .010$ & $+0.010 \pm .005$ & $+0.014 \pm .006$ & $+0.006 \pm .002$ \\
\bottomrule
\end{tabular}
\end{table*}

\begin{table}[h]
\caption{Score-inflation rates (relative score increase $>10\%$) showing model-specific exploitation risks. Per-row maxima are bolded.}
\label{tab:multi_prm_det}
\centering
\small
\resizebox{\linewidth}{!}{%
\begin{tabular}{lccc}
\toprule
\textbf{Model} & \textbf{Step infl.} & \textbf{Pos infl.} & \textbf{Conf infl.} \\
\midrule
Math-Shepherd & $22.4 \pm 3.1\%$ & $\mathbf{32.8 \pm 4.9\%}$ & $6.5 \pm 1.2\%$ \\
Qwen2.5-PRM & $\mathbf{47.6 \pm 4.3\%}$ & $21.5 \pm 3.7\%$ & $13.9 \pm 2.1\%$ \\
RLHFlow-PRM & $1.9 \pm 0.7\%$ & $1.4 \pm 0.6\%$ & $2.8 \pm 1.0\%$ \\
Skywork-PRM & $11.8 \pm 2.4\%$ & $\mathbf{30.2 \pm 4.7\%}$ & $9.5 \pm 1.8\%$ \\
DeepSeek-PRM & $\mathbf{38.8 \pm 3.9\%}$ & $15.1 \pm 2.8\%$ & $8.1 \pm 1.4\%$ \\
\bottomrule
\end{tabular}%
}
\end{table}
The two tables and the figure tell a consistent story: the vulnerability profile is model-specific rather than universal. Math-Shepherd and Skywork-PRM share a dominant failure mode in position-sensitivity, where reordering steps causes both the largest correlation drop and the highest inflation rate, suggesting that their per-step rewards mix step content with absolute step position more than the other models. Qwen2.5-Math-PRM and DeepSeek-PRM, by contrast, are most exposed to step-count manipulations under the tested aggregation, with Qwen2.5-PRM showing a striking 47.6\% inflation rate under step-inflation. This pattern of architecture-specific spurious correlations between PRM scores and superficial features of the reasoning chain is consistent with broader observations of spurious-feature reliance in reward models~\citep{zheng2025spurious} and with mechanistic studies of how reward models encode position-related signals~\citep{song2026mechanistic}. RLHFlow-PRM shows negative or near-zero collapse on all attacks, but its baseline correlation is also very low ($\rho_{\mathrm{base}} = 0.100$), so the result mostly reflects that there is little signal there to degrade. The LLM-as-critic control has substantially smaller correlation collapse across all three attacks, which suggests that a direct generative correctness judgment responds differently to step insertion and minor reordering than dense PRM scoring; this result should be interpreted as a control comparison rather than as evidence that zero-shot critics are generally preferable to PRMs in real pipelines. Additional visual analyses of attack effects, influence rates, model-specific robustness patterns, score inflation, mitigation behavior, and dataset filtering are provided in Appendix~\ref{app:supplementary_figures}.

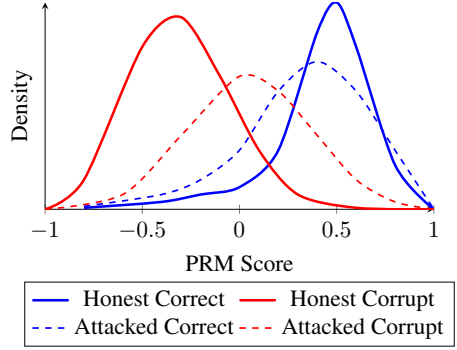
\begin{figure}[t]
\centering
\begin{tikzpicture}[scale=0.8]
\begin{axis}[
    width=8cm, height=5cm,
    enlargelimits=false,
    ylabel={Density},
    xlabel={PRM Score},
    legend style={at={(0.5,-0.35)}, anchor=north, legend columns=2},
    axis lines=left,
    ytick=\empty,
    xmin=-1, xmax=1
]
\addplot [very thick, blue, smooth] coordinates {(-0.8, 0.01) (-0.4, 0.05) (-0.2, 0.1) (0.0, 0.15) (0.2, 0.4) (0.4, 1.2) (0.5, 1.4) (0.6, 1.1) (0.8, 0.3) (1.0, 0.0)};
\addlegendentry{Honest Correct}
\addplot [very thick, red, smooth] coordinates {(-1.0, 0.0) (-0.8, 0.2) (-0.5, 1.1) (-0.3, 1.3) (-0.1, 0.9) (0.1, 0.4) (0.3, 0.1) (0.6, 0.01) (1.0, 0.0)};
\addlegendentry{Honest Corrupt}
\addplot [thick, blue, dashed, smooth] coordinates {(-0.8, 0.02) (-0.4, 0.1) (-0.2, 0.2) (0.0, 0.4) (0.2, 0.8) (0.4, 1.0) (0.6, 0.8) (0.8, 0.4) (1.0, 0.0)};
\addlegendentry{Attacked Correct}
\addplot [thick, red, dashed, smooth] coordinates {(-1.0, 0.0) (-0.6, 0.1) (-0.3, 0.5) (0.0, 0.9) (0.2, 0.8) (0.4, 0.5) (0.6, 0.2) (0.8, 0.05) (1.0, 0.0)};
\addlegendentry{Attacked Corrupt}
\end{axis}
\end{tikzpicture}
\caption{Empirical kernel-density estimates of Math-Shepherd score distributions for correct and corrupted chains before and after position-sensitivity attacks. Coordinates are exported from the released plotting script rather than hand-specified schematic curves. The attack blurs the separation between correct and corrupted chains by moving corrupted-chain scores toward the positive-reward region.}
\label{fig:distributions}
\end{figure}

A complementary view of the same effect appears in the score-distribution plot of Figure~\ref{fig:distributions}. Before the attack, correct and corrupted chains form well-separated modes; after the position attack, the corrupted distribution drifts toward higher scores while the correct distribution remains broadly anchored, and the two populations begin to overlap. The visual pattern is exactly what Theorem~\ref{thm:pos} predicts when a correctness-independent noise term is added to the score: covariance with $Y$ is preserved, variance grows, and correlation shrinks.

\section{Mitigations}
\label{sec:mitigations}

The vulnerability profiles in the previous section motivate three concrete defenses, each targeting a different aspect of the dense-reward attack surface. We evaluate them on a held-out calibration split of 1,000 honest chains used to parameterise decision thresholds, and we report results for the Math-Shepherd PRM where the attacks are most pronounced. Throughout this section, $S_P(x)$ denotes the primary PRM score and $\mathcal{A}(x)$ a transformed chain. 

The first defense, ensemble flagging, draws on the reward-ensembling intuition of \citet{rame2024rewarded} and couples the PRM with a lightweight LLM critic via a conservative minimum. The second defense, conformal abstention, treats the PRM as a black box and instead checks whether the transformed chain's score is consistent with the calibration distribution. The third defense, length-normalised flagging, targets step-inflation directly by penalising raw chain length in the scoring function.

Conformal abstention yields near-zero detection because the attacks do not push scores outside the upper tail of the honest distribution; they redistribute scores within the honest range. Length normalisation reduces step-inflation flags effectively but cannot address position-sensitivity, where chain length is unchanged. Ensembling remains the most balanced defense among the simple candidates tested here, with an AUROC of $0.72$, but absolute detection rates are still modest. None of the three simple defenses dominates: a deployment-ready mitigation likely needs to combine penalty-based and ensemble-based components, calibrated to the specific vulnerability profile of the PRM at hand. Further formal details and full results are provided in Appendix~\ref{app:mitigations}.
\section{Conclusion}
\label{sec:conclusion}

Process reward models are becoming load-bearing in language-model training, and the framework that depends on them rests on a trustworthiness assumption that should be tested under label-preserving transformations rather than assumed. The EST-PRM evaluation across 4{,}687 retained chains and five public process-reward-style models reveals PRM-specific vulnerability profiles: Math-Shepherd and Skywork-PRM are dominantly sensitive to position-sensitivity, while Qwen2.5-Math-PRM and DeepSeek-PRM are dominantly sensitive to step-inflation. The same PRM that wins on natural-input accuracy can therefore lose on operational robustness, and the loss is patterned rather than random.

These findings suggest that PRM deployment should treat label-preserving robustness as a first-class evaluation objective rather than relying only on natural-trace agreement.

\section*{Limitations}
\label{sec:limitations_future_work}

Our study has some limitations .First, the attack strategies considered in EST-PRM are intentionally simple and do not involve adaptive or jointly optimised policies. Consequently, more sophisticated adversaries that co-optimise transformations with reward maximisation may expose additional failure modes beyond those identified in this work.

Second, our evaluation is restricted to mathematical reasoning tasks. While this setting is standard for process reward models, it remains unclear whether the observed vulnerability profiles generalise to other domains, including code generation, dialogue reasoning, or multimodal tasks.

Third, human validation is limited to 500 sampled chains. Although this provides evidence of systematic degradation, particularly under position sensitive attacks (Table~\ref{tab:human_val}), a larger scale annotation study would improve the reliability and statistical strength of these findings.

Finally, the mitigation strategies considered in this work are simple and stationary, and do not incorporate learned or adaptive defense mechanisms that may offer stronger robustness through dynamic responses to structural perturbations.

These limitations leave open the need for stronger adaptive attack formulations, broader evaluation beyond mathematical reasoning, and learned or model-aware defenses that explicitly account for label-preserving transformations.

\bibliography{ref}

\appendix

\section{Extended Related Work}
\label{app:related_work}

Beyond direct PRM evaluation related work also studies proxy-gaming analyses,reliability of learned evaluators through calibration methods and reinforcement learning pipelines that depend on intermediate reward signals.

\paragraph{Reward hacking and proxy gaming.}
Reward hacking and proxy gaming have been studied extensively in learned reward functions and alignment objectives~\citep{amodei2016concrete,manheim2018goodhart,casper2023rlhf,anon2025proxygaming}. Benchmarking frameworks such as RewardBench~\citep{lambert2024rewardbench} test susceptibility to formatting artifacts, heuristic exploitation, and stylistic shortcuts at the level of scalar outcome rewards. These approaches rely on whole-response evaluation, where a single aggregated score or judgment applies to the full output. In contrast, PRMs define a structured scoring regime over intermediate reasoning steps, which yields a more fine-grained but also more fragile evaluation surface. EST-PRM extends proxy-gaming analysis to this dense-reward setting. It fixes final correctness and varies reasoning structure to isolate sensitivity of step-level scoring functions.
\paragraph{Calibration and uncertainty estimation.}
Calibration-based analyses of model confidence provide a related but distinct perspective on reward reliability. Classical calibration techniques such as temperature scaling and Platt scaling~\citep{platt1999probabilistic,guo2017calibration} measure misalignment between predicted probabilities and empirical correctness. More recent conformal prediction and uncertainty quantification methods for language models~\citep{quach2024cp,mohri2024factuality,kumar2023conformal_llm,shihab2025infolift} extend this framework to distribution-free uncertainty estimation. These methods focus on predictive reliability under natural data variation and do not address invariance under structure-preserving transformations. EST-PRM instead evaluates whether the scoring signal itself allows systematic manipulation without change in semantic correctness, which exposes a complementary failure mode not captured by calibration-based analyses.

\paragraph{RL pipelines, learned evaluators, and decoding.}
A broader line of work studies reinforcement learning pipelines with learned or hybrid reward signals, including self-distillation and sparse-reward augmentation frameworks~\citep{hubotter2025sdpo,wang2025openclawrl,ko2026reopold,yang2026rlsd,li2026srpo,kleinebuening2026userinteractions}, which build on the RLHF paradigm~\citep{ouyang2022instructgpt}. These systems rely on learned evaluators to provide dense or structured feedback during optimization. Test-time strategies such as process-supervised search~\citep{snell2024scaling} and verification-oriented decoding schemes~\citep{leviathan2023speculative,chen2023accelerating,li2024eagle,cai2024medusa,fu2024lookahead} rely on stability of intermediate scoring signals under benign transformations. As these assumptions become more central to training and inference pipelines, direct evaluation of robustness to label-preserving structural perturbations becomes necessary, which defines the focus of EST-PRM.

\section{Mitigation details and full results}
\label{app:mitigations}

\subsection{Formal definitions of mitigation methods}

\begin{itemize}

\item \textbf{Ensemble flagging.} The first defense couples the PRM with a lightweight LLM critic via a conservative minimum,
\[
S_{\mathrm{ens}}(x) = \min(S_P(x), S_{\mathrm{critic}}(x)).
\]
A chain is flagged as adversarial when
\[
S_{\mathrm{ens}}(\mathcal{A}(x)) > (1+\tau) S_P(x).
\]

\item \textbf{Conformal abstention.} The second defense follows the standard split-conformal construction~\citep{vovk2005algorithmic,angelopoulos2023gentle,angelopoulos2024crc} together with the Learn-Then-Test framework~\citep{angelopoulos2022ltt,bates2021selective}. We compute non-conformity scores
\[
E_j = -S_P(x_j)
\]
on the calibration set, define $q$ as the $(1-\alpha)$ quantile of $\{E_j\}$, and abstain when
\[
-S_P(\mathcal{A}(x)) > q.
\]

\item \textbf{Length-normalised flagging.} The third defense defines
\[
S_{\mathrm{LN}}(x) = S_P(x) - \lambda |x|,
\]
where $|x|$ is the step count. A chain is flagged when
\[
S_{\mathrm{LN}}(\mathcal{A}(x)) > (1+\tau) S_{\mathrm{LN}}(x).
\]

\end{itemize}
\subsection{Full quantitative results}

\begin{table}[h]
\caption{Mitigation behavior on Math-Shepherd. We report true positive rate (TPR) detection of attacked chains against false positive rate (FPR) on the honest distribution, alongside overall AUROC.}
\label{tab:mitigations}
\centering
\scriptsize
\resizebox{\linewidth}{!}{%
\begin{tabular}{p{2.0cm}ccc}
\toprule
\textbf{Mitigation} & \textbf{Step (TPR/FPR)} & \textbf{Pos (TPR/FPR)} & \textbf{AUROC} \\
\midrule
Ensemble flag & $0.19 / 0.05$ & $0.34 / 0.06$ & $0.72 \pm .02$ \\
Conformal abs. & $0.00 / 0.01$ & $0.00 / 0.01$ & $0.53 \pm .01$ \\
Length-norm flag & $0.13 / 0.04$ & $0.22 / 0.05$ & $0.63 \pm .02$ \\
\bottomrule
\end{tabular}%
}
\end{table}

Table~\ref{tab:mitigations} reports detection performance and false-positive behavior for $\alpha=0.05$ and $\lambda=0.02$ on the held-out calibration split of 1,000 honest chains. Conformal abstention shows near-zero detection because attacks remain within the calibration support. Length normalisation reduces step-inflation-related false positives but fails under structural rearrangements. Ensemble flagging achieves the best overall trade-off with AUROC $0.72$, though absolute detection rates remain limited.

\begin{table*}[t]
\caption{Model and scoring configuration used for the multi-model evaluation. ``Critic'' denotes the zero-shot LLM control rather than a trained PRM.}
\label{tab:model_config}
\centering
\scriptsize
\setlength{\tabcolsep}{3pt}
\renewcommand{\arraystretch}{1.15}
\begin{tabularx}{\textwidth}{p{2.3cm}p{3.4cm}p{1.6cm}X p{2.1cm}}
\toprule
\textbf{Model} & \textbf{Checkpoint identifier} & \textbf{Role} & \textbf{Scoring interface} & \textbf{Aggregation} \\
\midrule
Math-Shepherd-7B &
\texttt{math-shepherd-\allowbreak mistral-7b-prm} &
PRM &
native step tag with $+/-$ token probabilities &
mean over step scores \\

Qwen2.5-Math-PRM-7B &
\texttt{Qwen/\allowbreak Qwen2.5-Math-PRM-7B} &
PRM &
separator-token classification probability &
mean over step scores \\

RLHFlow-PRM-8B &
\texttt{RLHFlow/\allowbreak Llama3.1-8B-PRM} &
PRM &
native reward-head or token-score adapter &
mean over step scores \\

Skywork-o1-Open-PRM-8B &
\texttt{Skywork/\allowbreak Skywork-o1-Open-PRM-8B} &
PRM-style verifier &
process-score adapter supplied in artifact &
mean over step scores \\

DeepSeek-Math-7B-PRM &
\texttt{deepseek-math-7b-prm} &
PRM-style verifier &
process-score adapter supplied in artifact &
mean over step scores \\

Llama-3-8B-Critic &
\texttt{Meta-Llama-3-8B-Instruct} &
Critic control &
zero-shot binary correctness prompt &
scalar response score \\
\bottomrule
\end{tabularx}
\end{table*}

\section{Per-Dataset Vulnerability Breakdown}
\label{app:per_dataset}

Table~\ref{tab:dataset_counts} reports the exact sample counts after removing invalid or non-parsable chains, and Table~\ref{tab:dataset_breakdown} splits the performance of the dominant vulnerability modes by source dataset. Step-inflation remains stronger on GSM8K, where the average initial reasoning chain is shorter and inserted steps therefore carry more proportional weight. Position-sensitivity is particularly potent on MATH-500, where deeper reasoning graphs offer more valid permutations that structurally confuse the reward model.

\begin{table}[h]
\caption{Sample counts per dataset after filtering invalid or non-parsable chains. The main experiments use the 4{,}687 retained chains.}
\label{tab:dataset_counts}
\centering
\scriptsize
\begin{tabular}{lccc}
\toprule
\textbf{Split} & \textbf{MATH-500} & \textbf{GSM8K} &\textbf{PRMBench} \\
\midrule
Initial & 2{,}000 & 2{,}000 & 1{,}000 \\
Post-filter & 1{,}812 & 1{,}930 & 945 \\
\bottomrule
\end{tabular}
\end{table}

\begin{table}[h]
\caption{Correlation collapse ($\Delta\rho$) by dataset split. Reported for the dominant attack class per PRM.}
\label{tab:dataset_breakdown}
\centering
\scriptsize
\setlength{\tabcolsep}{3pt} 
\renewcommand{\arraystretch}{0.9} 
\resizebox{\linewidth}{!}{
\begin{tabular}{lccc}
\toprule
\textbf{PRM (Dominant Attack)} & \textbf{MATH-500} & \textbf{GSM8K} & \textbf{PRMBench} \\
\midrule
Math-Shep (Position) & $0.165 \pm .042$ & $0.138 \pm .034$ & $0.150 \pm .040$ \\
Qwen2.5 (Step-Infl) & $0.065 \pm .015$ & $0.108 \pm .021$ & $0.095 \pm .019$ \\
\bottomrule
\end{tabular}
}
\end{table}

\section{Experimental Prompts}
\label{app:prompts}

We used three prompt templates: one to generate initial chains, one to corrupt them with a controlled logical error, and one to query the LLM-as-critic baseline. The chain-generation prompt asked the model to ``solve the following math problem step-by-step, break your reasoning into distinct sentences or short mathematical steps, and end with your final answer enclosed in $\backslash$boxed\{\}'', followed by the question. The corruption prompt asked the model to ``modify exactly one step near the middle of the chain to introduce a subtle logical error, arithmetic mistake, or sign flip, ensure the subsequent steps logically follow from this new incorrect step, leading to a wrong final answer, and preserve the exact style, formatting, and tone of the original chain'', followed by the chain. The LLM-critic prompt asked the model to ``act as a strict mathematical reviewer'', to ``evaluate the following reasoning chain step-by-step'', and to ``output `Score: 1.0' if all steps are correct and `Score: 0.0' if there is a logical or arithmetic flaw anywhere in the chain'', followed by the chain. The exact strings used in the experiments are included verbatim in the artifact release.

\section{Annotation Guidelines}
\label{app:annotation}

Expert annotators were given a three-part rubric during the human validation phase. They first checked answer preservation, judging whether the final boxed answer in the modified chain algebraically matched the original chain. They then checked reasoning preservation, ensuring that no new mathematical errors had been introduced by the modification; for step-inflation they verified that any restatement was factually aligned with prior steps, and for position-sensitivity they verified that the reordering did not violate causal mathematical dependencies (for example, by using a variable before it was defined). Finally, any chain receiving a negative judgment in either category from a majority of the annotators was flagged as an invalid attack and removed from the benchmark. Annotators were blinded to attack type and PRM scores throughout.

\section{Attack Templates and Configuration}
\label{app:attack_templates}

Table~\ref{tab:attack_examples} shows what each transformation does on a concrete short example chain. The position attack only swaps steps whose dependency graph permits the reorder, and the final answer is preserved in every case by construction.

\begin{table*}[t]
\caption{Concrete illustrations of the three EST-PRM attacks on a short example chain. In every case the final answer ($\boxed{10}$) is preserved. The position attack only swaps steps whose dependency graph permits the reorder.}
\label{tab:attack_examples}
\centering
\small
\begin{tabular}{p{2.6cm}p{12.2cm}}
\toprule
\textbf{Attack} & \textbf{Example transformation (final answer preserved)} \\
\midrule
Honest chain & $s_1$: First, $a = 2 \cdot 3 = 6$. \quad $s_2$: Independently, $b = 5 - 1 = 4$. \quad $s_3$: Then $a + b = 10$. \quad $\boxed{10}$ \\
\addlinespace
Step-inflation & $s_1$: First, $a = 2 \cdot 3 = 6$. \quad \emph{[Let me restate: $a$ is now defined.]} \quad $s_2$: Independently, $b = 5 - 1 = 4$. \quad \emph{[To summarise: both $a$ and $b$ are computed.]} \quad $s_3$: Then $a + b = 10$. \quad $\boxed{10}$ \\
\addlinespace
Position-sensitivity & $s_2$: Independently, $b = 5 - 1 = 4$. \quad $s_1$: First, $a = 2 \cdot 3 = 6$. \quad $s_3$: Then $a + b = 10$. \quad $\boxed{10}$ \\
\addlinespace
Confidence-injection & $s_1$: \emph{Clearly,} $a = 2 \cdot 3 = 6$. \quad $s_2$: \emph{By definition,} $b = 5 - 1 = 4$. \quad $s_3$: \emph{It is immediate that} $a + b = 10$. \quad $\boxed{10}$ \\
\bottomrule
\end{tabular}
\end{table*}

Step-inflation attacks insert template restatement steps drawn from a small fixed pool, including ``Let me restate what we have so far: [paraphrase].'', ``To summarise the progress so far, [paraphrase].'', and ``Before continuing, I'll verify the current state: [paraphrase].'' Confidence-injection attacks prepend each step with a phrase drawn uniformly from the set $\{$``By a well-known result,'', ``It is immediate that,'', ``Clearly,'', ``Without loss of generality,'', ``By definition,'', ``Trivially,'', ``Obviously,''$\}$. Position-sensitivity attacks use three permutation strategies: reverse order when permitted by the dependency graph, front-loading the highest-PRM-reward eligible step, and back-loading the highest-PRM-reward eligible step. Dependency graphs are constructed using strict variable-and-definition tracking, so that a step defining a variable $x$ cannot be moved after a step consuming $x$. This dependency constraint is what produces the higher invalid-attack rate for position-sensitivity reported in Table~\ref{tab:human_val}: chains with tight dependency structures admit fewer valid permutations.

\section{Supplementary Figures}
\label{app:supplementary_figures}

This section provides additional visual analyses supporting the EST-PRM evaluation. The figures summarize the evaluation design, attack effects, influence rates, human validation results, model-specific robustness patterns, mitigation behavior, reproducibility components, and dataset filtering statistics.

Figure~\ref{fig:prm_eval_comparison} provides a high-level comparison between EST-PRM and prior reward-model evaluation benchmarks, emphasizing the role of PRM-specific and dense adversarial evaluation.

\begin{figure}[t]
\centering
\includegraphics[width=\columnwidth]{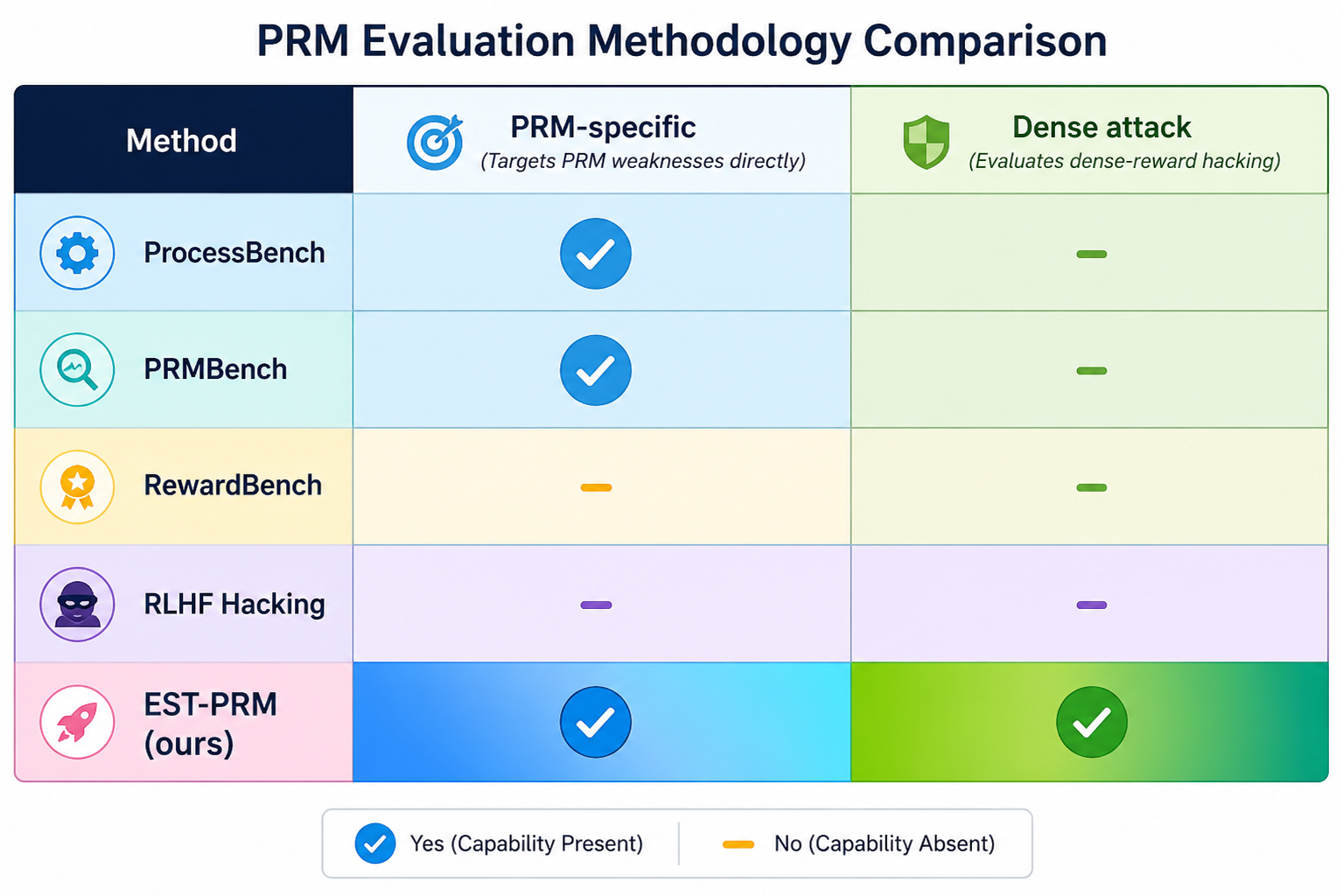}
\caption{Comparison of PRM evaluation methodologies across existing benchmarks and the proposed EST-PRM framework. The figure illustrates whether each method is PRM-specific and whether it incorporates dense adversarial attacks. Prior work, including ProcessBench, PRMBench, RewardBench, and RLHF Hacking, primarily evaluates models under natural or outcome-level errors with limited consideration of dense adversarial settings. In contrast, EST-PRM explicitly integrates both PRM-specific evaluation and dense adversarial perturbations (Step-inflation, Position, and Confidence attacks), enabling a more comprehensive assessment of reward model robustness under realistic and adversarial conditions.}
\label{fig:prm_eval_comparison}
\end{figure}

Figure~\ref{fig:signed_attack_effects} further examines how different attack types change model scores and influence estimates for Math-Shepherd-7B.

\begin{figure}[t]
\centering
\includegraphics[width=\columnwidth]{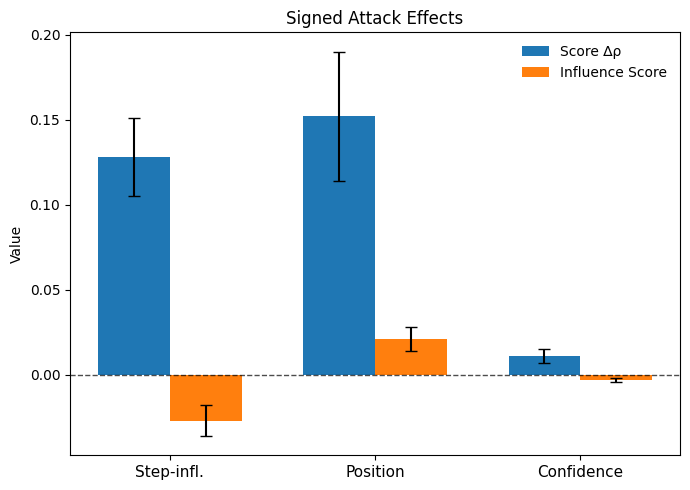}
\caption{Comparison of signed attack effects on the Math-Shepherd-7B model under Step-inflation, Position, and Confidence perturbations. The figure reports $\Delta\rho$ (Score) and Influence Score with 95\% bootstrap confidence intervals. Step-inflation and Position attacks produce larger, mixed-sign effects across both metrics, indicating unstable but substantial perturbations, while Confidence-based attacks remain consistently near zero, reflecting weak and largely ineffective influence on model behavior.}
\label{fig:signed_attack_effects}
\end{figure}

To complement the signed-effect analysis, Figure~\ref{fig:influence_rate} reports the influence rate induced by each attack strategy.

\begin{figure}[t]
\centering
\includegraphics[width=\columnwidth]{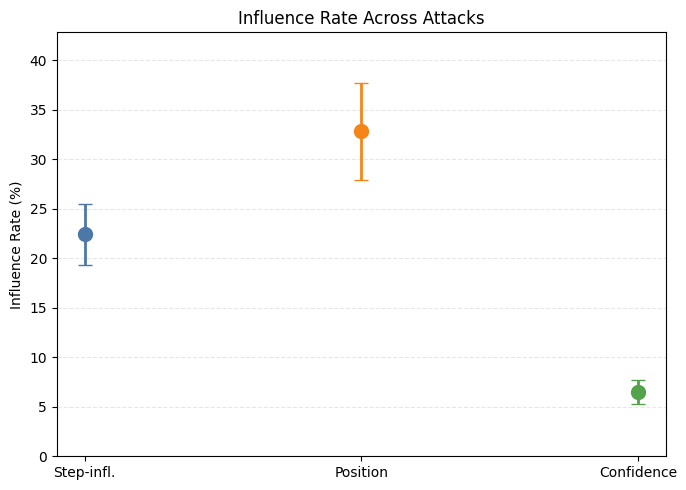}
\caption{Influence rates under different adversarial strategies (Step-inflation, Position, and Confidence) for the Math-Shepherd-7B model. Each point shows the mean influence rate, with vertical error bars representing 95\% bootstrap confidence intervals. Results indicate that Step-inflation and Position attacks consistently produce higher influence rates than Confidence-based attacks, suggesting greater model susceptibility to structural and positional perturbations compared to confidence perturbations.}
\label{fig:influence_rate}
\end{figure}

Figure~\ref{fig:human_validation} evaluates whether the proposed attacks preserve the original answer labels and reasoning consistency under human validation.

\begin{figure}[t]
\centering
\includegraphics[width=\columnwidth]{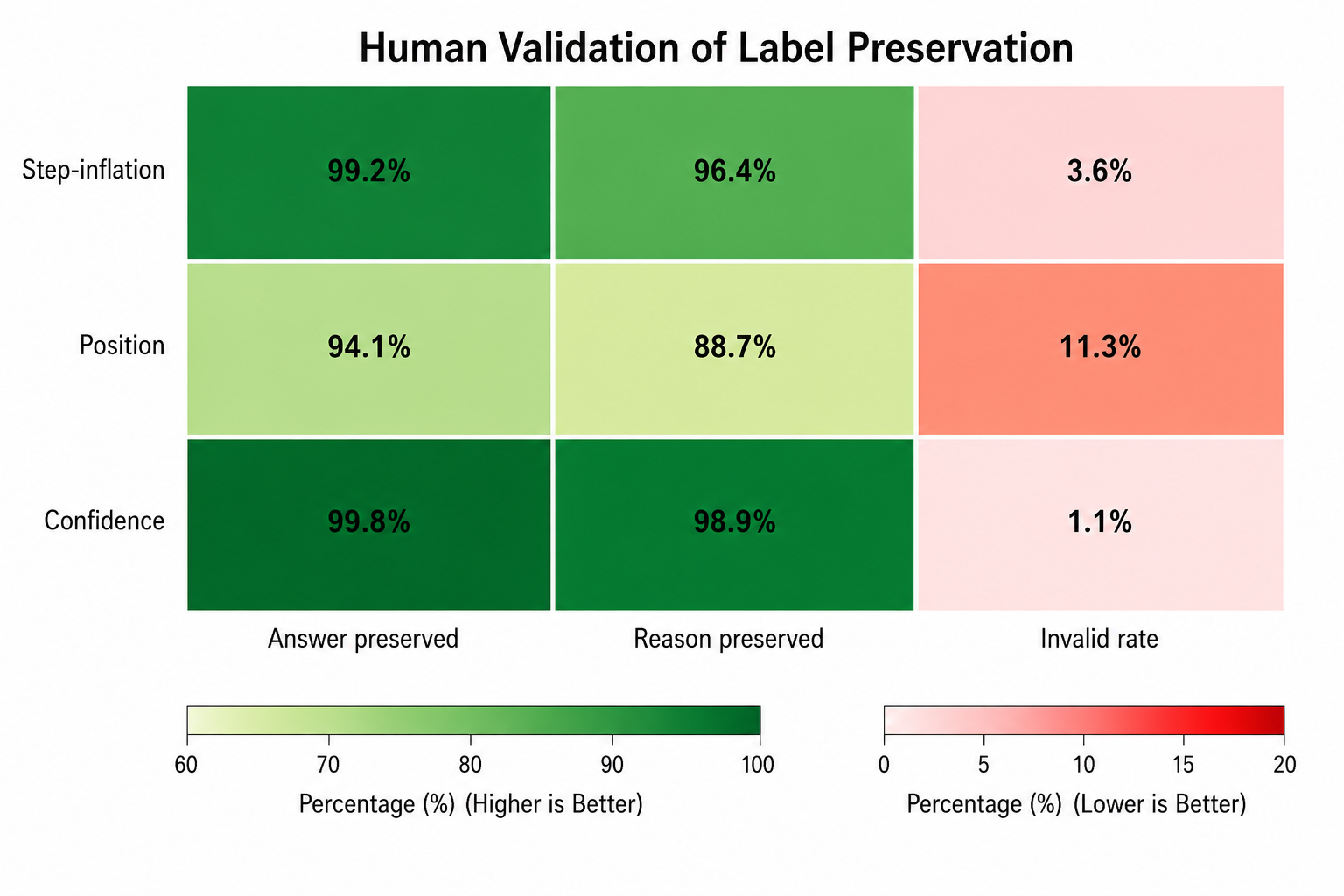}
\caption{Human validation of label preservation under different attack strategies. The heatmap shows the percentage of preserved answers and reasoning consistency alongside invalid generation rates across Step-inflation, Position, and Confidence attacks. Green intensity indicates higher preservation quality, while red intensity highlights higher failure rates, enabling a clear comparison of robustness across attack types on a 500-chain evaluation subset.}
\label{fig:human_validation}
\end{figure}

Figure~\ref{fig:correlation_collapse_models} extends the robustness analysis across multiple models and attack categories, showing model-specific correlation collapse patterns.

\begin{figure}[t]
\centering
\includegraphics[width=\columnwidth]{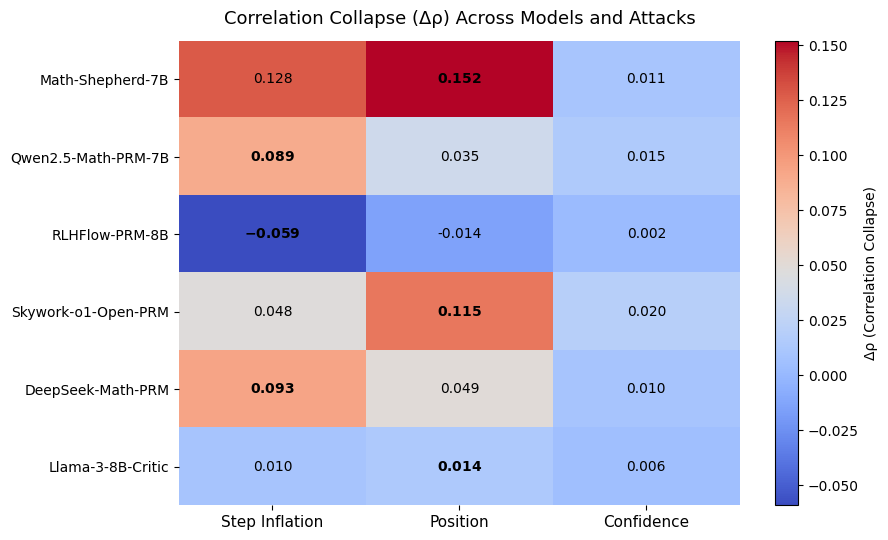}
\caption{Heatmap showing correlation collapse ($\Delta\rho$) for six models under three adversarial settings: Step Inflation, Position, and Confidence attacks. Each cell represents the change in correlation relative to the baseline, with color intensity encoding both magnitude and direction of $\Delta\rho$. Per-model dominant vulnerability is highlighted in bold, corresponding to the maximum absolute $\Delta\rho$ within each row. The visualization highlights model-specific sensitivity patterns, demonstrating that correlation collapse is not uniform across architectures, with Step Inflation and Position attacks generally inducing stronger disruptions than Confidence-based attacks.}
\label{fig:correlation_collapse_models}
\end{figure}

Figure~\ref{fig:score_inflation} presents complementary evidence on score inflation, identifying which attack type most strongly increases reward-model scores for each PRM.

\begin{figure}[t]
\centering
\includegraphics[width=\columnwidth]{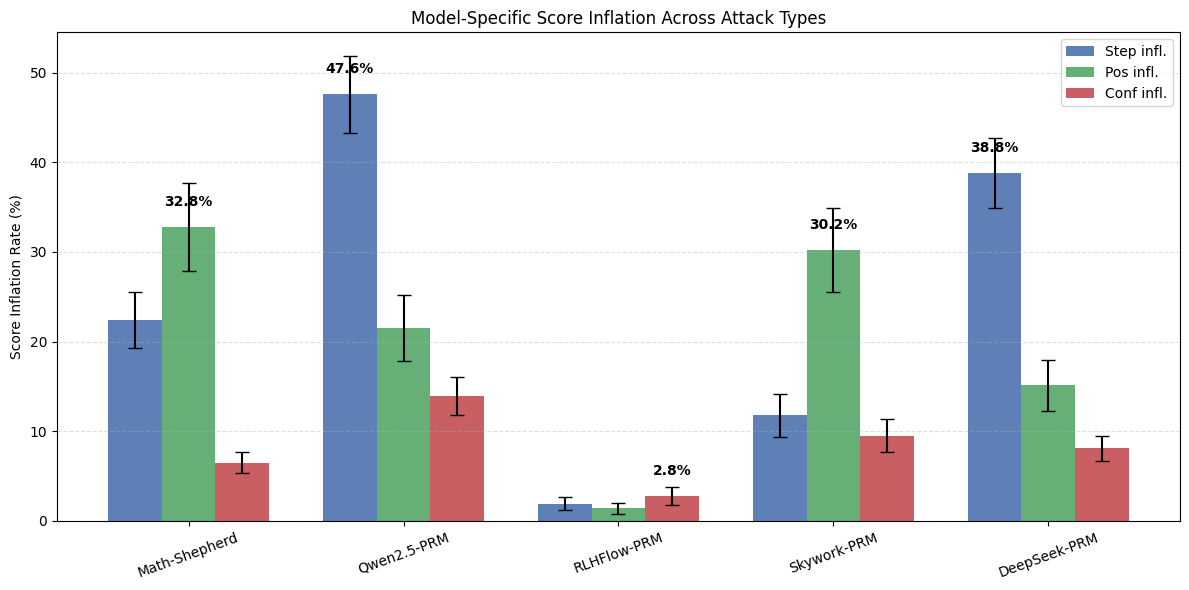}
\caption{Score-inflation rates (relative score increase $>$ 10\%) across five process reward models under Step, Position, and Confidence attacks. Bars show mean inflation percentages with error bars indicating uncertainty (± values). Per-row maxima are highlighted in bold to indicate the most vulnerable attack type for each model, revealing distinct model-specific exploitation patterns.}
\label{fig:score_inflation}
\end{figure}

Figure~\ref{fig:mitigation_behavior} summarizes the mitigation trade-off on Math-Shepherd by comparing detection performance across different mitigation strategies.

\begin{figure}[t]
\centering
\includegraphics[width=\columnwidth]{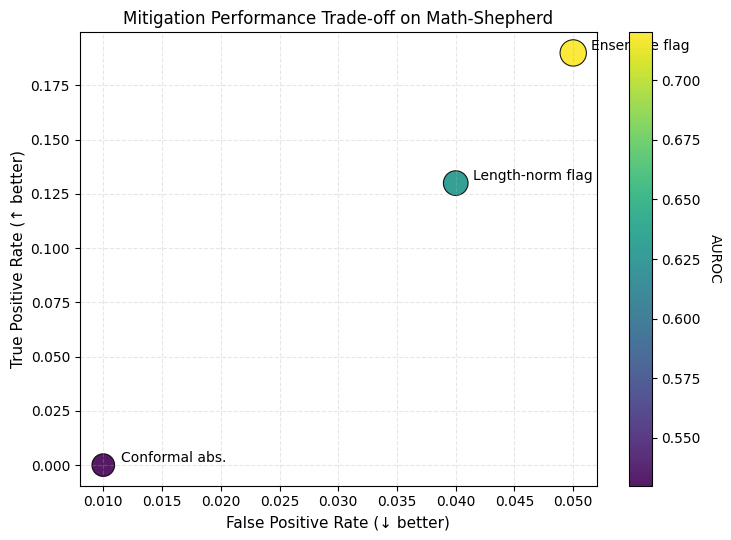}
\caption{Mitigation behavior on Math-Shepherd. The figure illustrates the detection performance of different mitigation strategies in terms of true positive rate (TPR) versus false positive rate (FPR), highlighting their trade-off under adversarial conditions. Marker size and color encode AUROC, which together provide a unified representation of overall discriminative effectiveness across methods.}
\label{fig:mitigation_behavior}
\end{figure}

Figure~\ref{fig:artifact_release} summarizes the EST-PRM artifact release components, including the major elements needed for reproducing the evaluation pipeline.
\begin{figure}[t]
\centering
\includegraphics[width=\columnwidth]{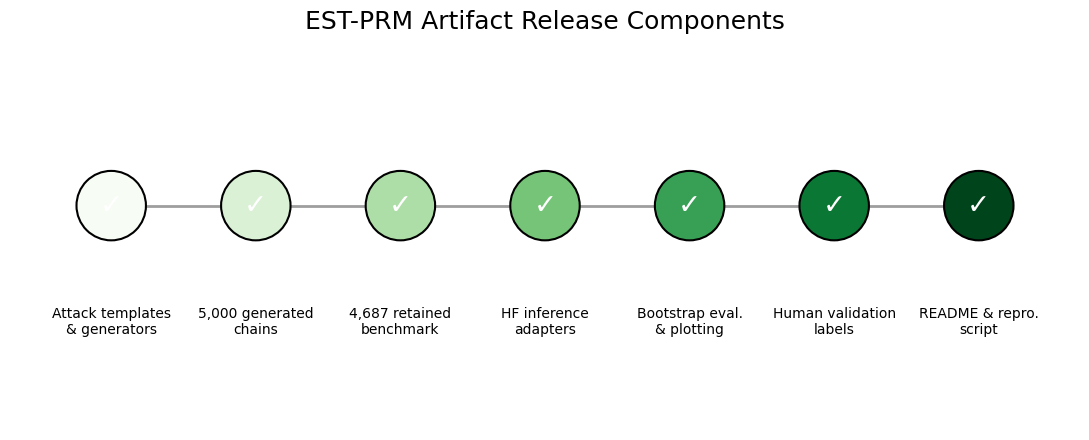}
\caption{Overview of the EST-PRM evaluation pipeline and reproducibility components used in the experiments.}
\label{fig:artifact_release}
\end{figure}

Figure~\ref{fig:dataset_filtering} shows how dataset filtering affects the final evaluation set across benchmarks.

\begin{figure*}[t]
\centering
\includegraphics[width=\columnwidth]{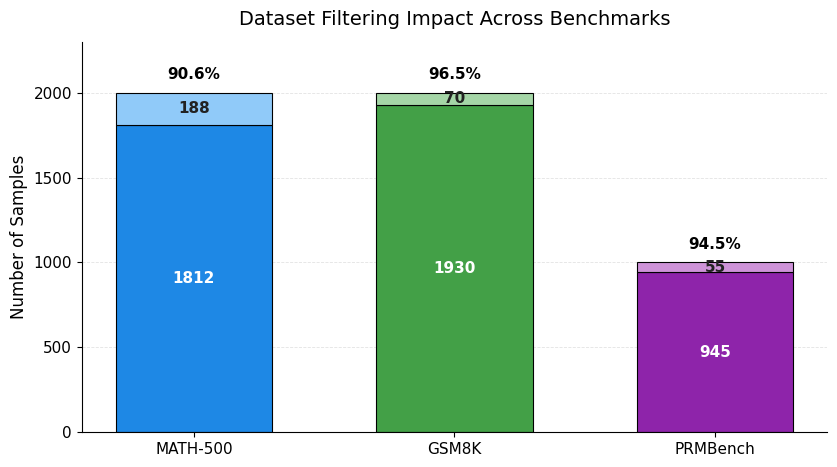}
\caption{The figure illustrates the effect of filtering invalid or non-parsable reasoning chains across MATH-500, GSM8K, and PRMBench, showing the reduction from initial samples to post-filter retained samples, along with the corresponding number of filtered-out instances and retention rates; the retained subset forms the final evaluation set used in the main experiments (4,687 chains in total). Filtered-out samples correspond to invalid or non-parsable chains.}
\label{fig:dataset_filtering}
\end{figure*}

Finally, Figure~\ref{fig:correlation_collapse_datasets} reports correlation collapse across datasets and PRMs under their dominant attack settings.
\begin{figure*}[t]
\centering
\includegraphics[width=\columnwidth]{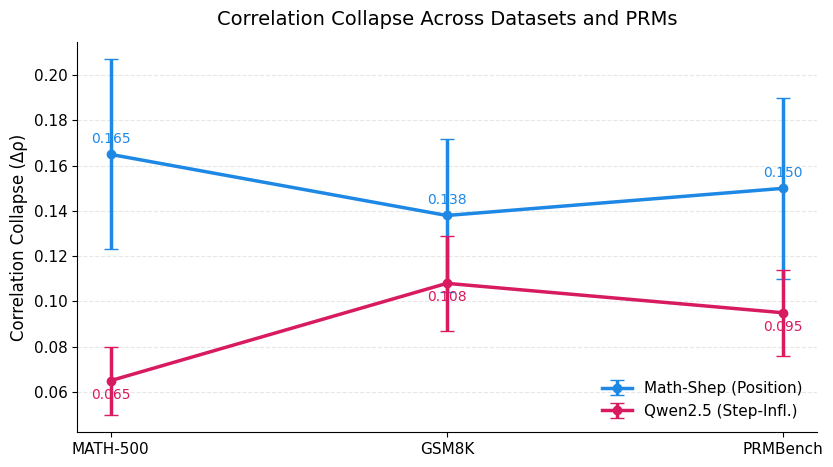}
\caption{The figure presents the correlation collapse ($\Delta\rho$) across MATH-500, GSM8K, and PRMBench for two PRMs under their dominant attack settings, namely Math-Shep (Position attack) and Qwen2.5 (Step-Inflation attack). Mean $\Delta\rho$ values are reported with standard deviation (±) as error bars, highlighting variability across bootstrap runs and demonstrating consistent differences in robustness patterns across datasets and models.}
\label{fig:correlation_collapse_datasets}
\end{figure*}
\end{document}